\documentclass{article}

\usepackage[T1]{fontenc}
\usepackage[utf8]{inputenc}
\usepackage{lmodern}
\usepackage{microtype}

\usepackage{amsmath}
\usepackage{amssymb}
\usepackage{amsthm}
\usepackage{mathtools}
\usepackage{enumerate}
\usepackage[dvipsnames]{xcolor}

\usepackage{multirow}
\usepackage{empheq}
\usepackage{listings}
\usepackage{color}

\usepackage[square,comma,numbers]{natbib}

\usepackage{hyperref}

\numberwithin{equation}{section}
\def\reff#1{{\rm(\ref{#1})}}

\def\be\begin{equation}
\def\ee\end{equation}

\def\bea\begin{eqnarray}
\def\eea\end{eqnarray}

\def\bea*\begin{eqnarray*}
\def\eea*\end{eqnarray*}

\def\cA{{\mathcal A}}
\def\cB{{\mathcal B}}
\def\cC{{\mathcal C}}

\def\cF{{\mathcal F}}
\def\cG{{\mathcal G}}

\def\cK{{\mathcal K}}
\def\cL{{\mathcal L}}

\def\cN{{\mathcal N}}
\def\cO{{\mathcal O}}

\def\cS{{\mathcal S}}

\def\cT{{\mathcal T}}

\def\cX{{\mathcal X}}
\def\cY{{\mathcal Y}}
\def\cZ{{\mathcal Z}}

\def\E{\mathbb{E}}
\def\F{\mathbb{F}}

\def\N{\mathbb{N}}

\def\P{\mathbb{P}}
\def\R{\mathbb{R}}

\def\eps{\epsilon}

\def\reff#1{{\rm(\ref{#1})}}
\def\Om{\Omega}

\def\oell{\overline{\ell}}
\newcommand{\bigO}{O}

\theoremstyle{plain}
\newtheorem{theorem}{Theorem}[section]
\newtheorem{lemma}[theorem]{Lemma}
\newtheorem{corollary}[theorem]{Corollary}

\newtheorem{assumption}[theorem]{Assumption}
\newtheorem{definition}[theorem]{Definition}
\newtheorem{example}[theorem]{Example}
\newtheorem{remark}[theorem]{Remark}

\def\Xa{X^a}

\def\xa{x^a}

\def\Zi{Z^{(i)}}

\def\Xi{X^{(i)}}

\def\ase{\alpha^{\eps}}

\def\d{\mathrm{d}}

\newcommand{\ceq}{\mathrm{ce}}

\usepackage{fancyhdr}
\title{
Deep Empirical Risk Minimization in finance:\\
looking into the future}

\author{A.\ Max Reppen\footnote{Questrom School of Business, Boston University, Boston, MA, 02215, USA, email: {\texttt{amreppen@bu.edu}}.
Reppen was partly supported by the Swiss National Science Foundation grant SNF 181815.}
\and H. Mete Soner\footnote{Department of Operations Research and Financial
Engineering, Princeton University, Princeton, NJ, 08540, USA, email: 
{\texttt{ soner@princeton.edu}}. Research of Soner
 was partially supported by the National Science Foundation grant
 DMS 2106462.}}

\date{\today}

\begin{document}

\maketitle

\begin{abstract}
\noindent
Many modern computational approaches to 
classical problems in quantitative finance
are formulated
as empirical loss minimization (ERM), 
allowing direct applications of classical results 
from statistical machine learning.
These methods, designed to
directly construct the optimal feedback representation of
hedging or investment decisions,
are analyzed in this framework demonstrating
their effectiveness  as well as their
 susceptibility to generalization error.
 Use of classical techniques 
 shows that over-training renders
 trained investment decisions to become anticipative,
 and proves overlearning for large hypothesis spaces.
On the other hand, non-asymptotic estimates based on Rademacher complexity
show the convergence 
for sufficiently large training sets.
 These results emphasize the importance
 of synthetic data generation and 
 the appropriate calibration of complex models
 to market data.
 A numerically studied
stylized example illustrates these possibilities, including the importance of
problem dimension in the degree of overlearning,
and the  effectiveness of this approach.

 \end{abstract}


\section{Introduction}

Recent advances in the training of neural networks make
high-dimensional numerical studies  feasible for 
decisions under uncertainty, and in particular,
classical hedging and pricing problems in quantitative 
finance.  Although simulation-based methods
have been 
widely used in stochastic  (optimal) control for several decades 
\citep{BT},
only recently  \cite{HE},
 \cite{HEJ}  combine it with deep neural networks 
for the offline construction of 
optimal feedback actions
for sequential decision problems.
This approach,
which we call (\emph{dynamic}) \emph{deep empirical risk minimization} (ERM), 
assumes that a training set 
is either readily available or can be simulated through an assumed model,
and then an appropriate empirical average over this  training data is used
to construct a loss function
to be minimized over the 
network parameters.
A near-minimizer 
is the trained
network that approximates the
optimal investment actions.   
This technique is remarkably flexible and  tractable. It
can handle complex realistic
dynamics with ease, 
does not require a Markov structure,
and can be completely data-driven  
when a sufficiently large  training set is available.

Readily this method has been adopted in many
studies \cite{BHLP,BCJ,ETH,BGTW,BGTWM,GMS,HL,HPBL,RW}
to study various problems in quantitative finance and we discuss
them later in this introduction.  
We also refer the readers to the recent excellent surveys \cite{FMW,HPW,RWs}
and the references therein for more information.  
Our goal is to provide both a computational and a
theoretical assessment of this promising 
new methodology.  We use the setting of general stochastic  control
for optimal sequential investments
to study them  in a unified manner.

As always, the generalization of the trained network
is the key property that is of paramount importance.
We study this central question 
in a general setting with a fixed training set
to also
explore the feasibility
of purely data-driven implementations.
Our key observation is 
to carefully formulate the general
stochastic optimal control  problem by expressing
the pathwise cost of a feedback policy
as a function
of the randomness whose expected value
is  the performance of this strategy, so that
the loss function used in the training is 
its empirical average. A central structural assumption
this approach makes is that, 
the reward function for all actions 
can be computed once a
trajectory of the randomness is given. This 
structure which is 
pervasive in quantitative finance,
 is further elaborated in Remark \ref{r.structure}
below.
Thus,  these recent
approaches 
can be viewed as empirical risk minimization, enabling us
to directly apply classical  results from statistical machine learning
with ease
and  provide useful estimates 
and tools for their analysis.
The explicit dependence on the randomness 
also clarifies how overfitting causes 
non-adaptedness on the training data.
To  articulate this maybe not so well-known but potentially
conducive reformulation with clarity,
we avoid technical constructs and assumptions, 
and focus on 
emphasizing
the fundamental structures and connections.

It is well-known that optimal feedback controls or
investment strategies are determined by
the conditional expectations of the 
value function evaluated at  future controlled
random states  \citep{FS},
and the above method essentially uses a regression
estimate  of these  expectations.
It is therefore natural that it
implicitly faces the classical bias-variance
trade-off as articulated in  \cite{GBD,JWHT,LG}.
Indeed, optimal decisions 
depend on a time-varying estimate of the randomness
driving the dynamics of the state, which is 
basically the return process in finance,
and in some applications the available
training data is limited in size.
Hence, as opposed to interesting
recent studies \cite{BHMM,zhang}
arguing the benefits of more complex networks and interpolation,
 in this dynamic setting,
overfitting causes the loss of the most salient
restriction of the problem,  namely, adaptedness 
of decisions to the information flow.  

The above numerical procedure that we outline in Section \ref{s.algo},
uses a hypothesis space $\cN_k$ of feedback actions
and minimizes the empirical loss function in this space.
Therefore, the training is among adapted
processes,
as feedback controls use only the current information.
Still,
the global minima of the empirical loss functions
are achieved by actions whose coefficients
depend on the whole random path including the future,
as proved in Theorem \ref{t.ant}.
Therefore, as shown in Section~\ref{s.overlearning},
sufficiently large networks  may at every time
\emph{overlearn} the future randomness instead of estimating them,
generating
output investment decisions that anticipate the future. 
This renders the trained feedback actions  \emph{on the training set}
to be {\emph{non-adapted}} to the filtration
generated by the observable variables, 
and thus in-sample, they overperform the original control problem
 by implicitly circumventing the
essential restriction of the adaptedness of the decisions.
Consequently, feedback actions constructed by
 large enough {hypothesis spaces}  do not always
generalize and perform poorly out-of-sample.
Examples of Section~\ref{ss.examples}
clearly illustrate the concept of overlearning
and the consequent  non-adaptedness
of  the actions  in non-technical settings.
Theorem \ref{t.ant} considers the over-parametrized limit 
of the trained actions or,
equivalently, the limit as the hypothesis 
spaces gets denser.
It is shown that in the limit
the performance of naturally adapted trained feedback controls
is equal  to the performance
of the strictly smaller value
given by the minimization over the anticipative controls.
\newpage

In data-driven applications, 
this capability of the artificial networks to overlearn,
necessitates 
effective data enrichment when 
the available set is not sufficiently large.
As these methods are data-hungry,
this remark applies to a large class of problems 
coming from quantitative finance.
The recent survey \cite{A}
provides a thorough introduction
to this important question and outlines numerous
approaches from statistics.  
A related problem is the calibration
of complex models to market data.
Modern optimization tools can also be utilized in this context allowing 
 us to use
more complex models and even artificial neural
networks to achieve this goal.
Currently there is extensive and far-reaching research 
on this topic.  An example is the exciting recent
paper \cite{C} that use this approach to calibrate
a local volatility model to market data.
For further information we refer  the reader to
the references in \cite{C}.  Regularization
is also widely used to reduce over-learning
as we discuss further in Remark \ref{r.regular},
below.  Finally,
even with these, 
when overlearning renders the data need computationally infeasible,
one must use alternate methods such as the 
dynamic programming based \emph{Deep Galerkin} \cite{SS} method
widely used for Markovian models.

Despite this potential hurdle of overlearning, 
many papers that are closely related to
dynamic deep empirical risk minimization
report impressive numerical results in problems 
with a large number of states.
Central financial problems of hedging and portfolio management are
the  foci of the pioneering papers \cite{BGTW,BGTWM}.
These studies
demonstrate that complex and high-dimensional
stock dynamics,
and also market details such as transaction costs and 
market impact are within easy reach of this approach. 
In particular, they study multidimensional factor  models, 
delivering convincing evidence for the flexibility and the scope of the algorithm, 
particularly in high-dimensions.
\cite{BHLP,HPBL} develop several effective
algorithms including the one discussed
here as well as hybrid ones that use the 
Markovian structure together with dynamic programming.
The convergence analysis of \cite{BHLP}
complements the 
high-dimensional numerical experiments of \cite{HPBL}.
The recent paper \cite{Warin} provides 
an extensive and informative numerical analysis of the impact of 
the network architecture.
Another exciting series of papers 
\cite{BCJ,ETH} consider the difficult problem of optimal
stopping.  These papers solve numerous examples of practical interest
 in dimensions up to 100 and 
show that dynamic deep empirical risk minimization
yields feedback actions achieving values very close to 
the upper bounds  computable through their duals.
The computation of the free boundary is studied in
\cite{DeepStopping}. 
For further studies and more information, we refer the reader to the excellent 
survey papers \cite{FMW,HPW,RWs} and the references therein.

The expediency of our formulation reveals itself by enabling a plethora 
of techniques from statistical machine learning  \citep{SSB}, as evidenced by 
the non-asymptotic 
upper bounds proved in Section \ref{s.estimate} via empirical 
 Rademacher complexity \citep{BBL4,BBL3,KP}. 
In particular, Theorem \ref{t.rademacher} and the performance error  
estimate \eqref{e.diff} are analytic  manifestations of the bias-variance, 
or more precisely,  bias-complexity  trade-off (cf.~\cite{SSB}, Chp.~5)
in this context. 
 Indeed, for fixed training sets the complexity increases with 
 the size of the networks allowing for possible generalization error. 
On the other hand,
as the complexity of a fixed network
gets smaller with larger training data, for an
appropriate combination 
of network structure and data,
so does overlearning.
Specifically,
Corollary \ref{c.convergence}
shows that for sufficiently large training sets, 
actions constructed by 
appropriately wide or deep neural networks
are close to the desired solutions and
overlearning is negligible. 
Hence, through classical concepts,
we obtain  efficient estimates 
yielding structured convergence proofs.
The style of these results complements the
comprehensive analysis of  
\cite{HPBL} 
 for controlled Markov processes
and the convergence analysis
of dynamic deep empirical risk minimization
carried out in \cite{HLo}  for
backward stochastic differential equations.

Our numerical experiments{\footnote{These
experiments were carried out on personal computers.  The code and the logs, 
including random number generator seeds, 
are available at 
\url{https://gitlab.com/mreppen/dderm}
for their full reproduction.}} 
support these
theoretical observations as well.
In Section \ref{s.numerics}, 
we analyze a stylized Merton
utility maximization problem of  Example \ref{ex.utility}.
Like the previous papers, our results also  demonstrate the 
effectiveness of dynamic deep empirical risk minimization
in handling high-dimensional problems 
and the convergence of the algorithm.
Potential overlearning 
and its clear dependence  on the dimension of the 
randomness driving the dynamics
is shown
by comparing the in-sample and out-of-sample
performance of the 
trained networks.
Although we employ
early stopping based on out-of-sample performance,
there is  
always some amount of
overlearning whose level  strongly depends on the state dimension.
Our experiments with a training set of size  $100,000$
and three  hidden layers of width $10$
show  in-sample to out-of-sample performance differences
ranging from 1.5\% in $10$ dimensions
to 24\% in $100$ dimensions.
This dependence 
is further corroborated by experiments controlling for the number of network parameters.
Moreover, more aggressive minimization---beyond our conservative early stopping---results in substantial overlearning. 
In 100 dimensions,
the trained network soon reaches 30\% over-performance
over the known true solution 
in 100--200 epochs, and more would be possible with further iterations.
In these cases, the out-of-sample performance deteriorates
rapidly.

It is well documented 
in the literature and proven by our estimates
that the size of the training
data is central to the performance of this approach.  
In 100 dimensions, we achieve a 
remarkable improvement in the accuracy
of our numerical computations by increasing the
size of training data.
In studies in which data is simulated
from a model, one does not create an initial training set,
but rather simulate new data for each batch, essentially
creating a large enough training set to obtain accurate results.
Therefore, as discussed earlier, simulation ability that is consistent
with the market data is key for this method.

Structurally we assume  that the random
process driving the state is uncontrolled, which is the case when
it is given by the stock returns.
Although most control problems are 
formulated differently, if their dynamics is known,
with little effort many can reformulated to have this structure.
This is demonstrated for  time-discretized
controlled
diffusions in Section~\ref{ss.feedback}
and for all  Markov decision processes
in Appendix \ref{appendixB}.

The paper is organized as follows.  The  problem is
defined in Section~\ref{s.problem}
and reformulated  in Section~\ref{s.rp}.
Section~\ref{s.algo} describes dynamic 
the deep empirical risk minimization.
Two motivating
examples are given in Section~\ref{s.examples}
and overlearning is introduced and proved in Section~\ref{s.overlearning}.
Error estimates based on Rademacher complexity 
and convergence are proved in Section~\ref{s.estimate}.
Section~\ref{s.numerics} outlines the specifics of the network
structure, the optimization algorithm, and the experiments.
After concluding remarks,  Appendix \ref{appendixA} provides
a generalization of the overlearning theorem
{and Appendix \ref{appendixB} formulates the classical Markov decision problem 
in our framework}.
\vskip 0.2in

\section{Notation and Conventions}
\label{ss.notation}
This is a brief summary of our notation and conventions
for quick reference.
Precise definitions are given in subsequent sections.

Whenever possible, we use capital letters for random variables
(with exceptions for the time of maturity $T$ and the utility function $U$),
lower case letters for deterministic quantities,
and sets are denoted by calligraphic letters.
In particular, action space $\cA$, state space $\cX$, and 
perturbation space $\cZ$  are closed subsets of Euclidean spaces
with the usual Euclidean norm.
We assume all functions defined on these sets to be continuous. 
For a set $\cY$ and a positive integer $t$,
$\cY^t$ is the Cartesian product of $t$ copies of $\cY$.

\emph{Maturity} is a finite positive integer $T$.
A stochastic process $Y$ taking values in $\cY$ is a 
finite sequence of $\cY$-valued random variables $(Y_0, Y_1, \dots, Y_T)$.
A \emph{trajectory} represents one realization of this process and
is a deterministic sequence
$y \in \cY^{T+1}$.
We use parenthesized subscripts to denote the sequence up to a time $t$:
$$
y_{(t)}:=(y_0,\ldots,y_t) \in \cY^{t+1}\quad
\text{or}
\quad
Y_{(t)}:=(Y_0,\ldots,Y_t).
$$

{A  \emph{feedback action} is a Borel-measurable map 
$a: \cT \times \cX \times \cZ \to \cA$, and for such $a$,}
the corresponding controlled state is denoted by $X^a\in\cX^{T+1}$,
with initial condition 
$X^a_0=x\in\cX$.
The  initial condition $x$ is considered fixed,
and is included in $X^a_{(t)}=(x,X^a_1,\dots,X^a_t)\in\cX^{t+1}$, 
but  otherwise it is omitted in the notation.
The set of all bounded, continuous feedback actions is denoted by $\cC$.
The set $\cB$ of all Borel measurable functions $g :\cZ^T \to \cA^T$ is
related to \emph{anticipative controls}.

\section{Optimal Feedback Controls}
\label{s.problem} 
All investments problems we consider
can be formulated as
sequential 
decision problems under uncertainty,
or, equivalently, stochastic optimal control
problems in discrete-time with a finite horizon of $T$.
Thus, we study this more general problem
in which
actions are taken at  time points in the set
$$
\cT:=  \{0,1,\ldots,T-1\}.
$$
We assume that a stochastic process $Z$,
which is the stock returns in most applications,
drives the dynamics 
of the problem.  Each component $Z_t$  is a random variable on a probability space $\Om$
taking values in $\cZ$, a closed subset of a Euclidean space.  
We always set $Z_0=0$  and for 
$t>0$, with abuse of notation, we write
$$
Z_{(t)}:=(Z_1,\ldots,Z_t), \quad
\text{and}
\quad
Z:=(Z_1,\ldots,Z_T).
$$ 
Let $\F=(\cF_t)_{t=0,1,\ldots,T}$ be the
filtration generated by the process $Z$, i.e., for {$t=1,\ldots, T$,}
 $\cF_t$ is the smallest sigma-algebra
 so that 
 $Z_{(t)}  :  (\Om,\cF_t) \to \cZ^{t}$ 
 is Borel measurable, and $\cF_0=\{\emptyset,\Om\}$ is trivial.
 
 \subsection{Dynamics and Performance}
\label{ss.dyn}
At times $ t \in \cT$, investors 
choose an action $A_t$ with values in $\cA$,
a closed subset of a Euclidean space.
The control process $A=(A_0,\ldots,A_{T-1})$ is adapted
to the filtration $\F$, and the resulting state process denoted by $X^A$
takes values in the \emph{state-space} $\cX$, 
again a closed subset of a Euclidean space.
In the applications, the mark-to-market value
of the portfolio is always included in the state
as well as other relevant quantities depending 
on the modeling.
Given an initial condition $X^A_0=x \in \cX$, 
the controlled state $X^A$ solves
the simple random difference equation
$$
X^A_{t+1} =F_{t+1}(X^A_{(t)},A_t), \qquad  t \in \cT,
$$
where  $X^A_{(t)}=(X^A_0,\ldots,X^A_t)$,
and  $F_{t+1} : \Om \times \cX^{t+1} \times \cA \mapsto \cX$ 
is  $\cF_{t+1}$-measurable (as customary, we 
use Borel
subsets of the Euclidean spaces).
 The performance of the action $A$ is measured by,
\begin{equation}\label{e.vhat}
v(A):= \E[ \Phi(X^A,A) ],
\end{equation}
where {$\Phi : \Om \times \cX^{T} \times \cA^T \mapsto \R$
is $\cF_T$-measurable.}
The  optimization problem is to minimize the above performance function over 
a class of feedback actions discussed in the next section.

As $\cF_t$ is generated by 
$Z_{(t)}$, the dependence of any $\cF_t$ measurable random variable 
on the randomness is 
given entirely through $Z_{(t)}$.
Hence, there are functions,
$$
f(t,\cdot): \cX^{t+1} \times \cA  \times \cZ^{t+1} \to \cX,
\quad
\text{and}
\quad
\varphi: \cX^{T+1} \times  \cA^T \times \cZ^T \to \R,
$$
so that
$$
F_{t+1}(X^A_{(t)},A_t)=f(t,X^A_{(t)},A_t,Z_{(t+1)}),
\quad 
\text{and}
\quad
 \Phi(X^A,A)=\varphi(X^A,A,Z).
 $$
 
\subsection{Admissible Actions}
 \label{ss.feedback}
We  restrict the investment decisions to be a
 function of the current state, hence, a \emph{feedback action}.  That is,
 the investors determine their actions through  a
bounded, continuous function of their choice
$$
a: \cT \times \cX \times \cZ \to \cA.
$$
Indeed, for a chosen
feedback function $a$, one first
recursively defines a process $X^a$ by the equations,
\begin{equation}
\label{e.state}
X^a_0=x,\quad
\text{and}
\quad
X^a_{t+1} =f(t,X^a_{(t)},a(t,X^a_t,Z_t),Z_{(t+1)}), \qquad  t \in \cT.
\end{equation}
Then, the corresponding  control process is given by $A^a_t:=a(t,X^a_t,Z_t)$.
It is clear that $A^a$ is adapted to $\F$
and the process $X^a$ is equal to the state process $X^{A^a}$
given by the control process $A^a$.  Hence, $A^a_t=a(t,X^{A^a}_t,Z_t)$.
Let $\cC$ be the set of all bounded, continuous {\emph{feedback actions}}
$a: \cT \times \cX \times \cZ \to \cA$.
In our notation, we use $A^a$ and $a$ interchangeably and write
$$
 v(a):=v(A^a)=
 \E[ \varphi(X^a,A^a,Z) ].
$$
Feedback actions
are easily implementable and are therefore widely used in 
practice.

\subsection{Problem}
\label{ss.dp}

The stochastic optimal control problem in discrete time---or an investment problem---is to
 \begin{equation}
 \label{e.prob}
 \text{minimize}\ a \in \cC \ \mapsto \ 
 v(a)= \E[ \varphi(X^a,A^a,Z) ],
\end{equation}
where $X^a$ is the solution of \reff{e.state}
and the expectation is over the distribution of $Z$.
We assume that $f(t,\cdot): \cX^{t+1} \times \cA  \times \cZ^{t+1} \to \cX$
determining the dynamics in \reff{e.state} and
the cost function 
$\varphi: \cX^T \times \cA^T \times  \cZ^T \to \R$  
in \reff{e.prob}
are  given and known. 

\begin{remark}
\label{r.markov}
{\rm{For Markovian models, the restriction to feedback controls 
causes no loss of generality.
Indeed, let $\cA_{ad}$ be
the set of all adapted and bounded processes $A\in \cA^T$, and
suppose that 
$Z$ is a Markov process, $f(t,\cdot)$  depends 
only on  $Z_{t+1}$ and not on $Z_{(t+1)}$,
and  $\varphi$ is given by
$$
\varphi(X^A,A,Z)= \hat{\varphi}(X^A_T)+
\sum_{t \in \cT}\ \psi(t,X^A_t,A_t,Z_t),
$$
for some {given} functions $\psi$ and $\hat{\varphi}$.
Then, under reasonable assumptions, one can 
show that among all
adapted processes there are near-maximizers that are of feedback form.
}}\end{remark}

\begin{remark}
\label{r.classical}
{\rm{The above model can also be 
obtained as an appropriate discretization of 
finite-horizon, continuous-time
problems. As an example,  consider
the classical optimal control of diffusion processes
with  a finite horizon $T_0$
with the admissible controls $\widehat{\cA}_{ad}$
given as the set of all adapted 
and bounded processes $A:[0,T_0] \to \cA$.
Then, the problem is to
$$
\text{minimize}\ A \in \widehat{\cA}_{ad} \mapsto \ v(A):=
\E \bigg[ \int_0^{T_0} \tilde{\psi}(u,\widehat{X}^{A}_u,A_u)\d u + 
\tilde{\varphi}(\widehat{X}^{A}_{T_0})
\bigg],
$$
subject to  dynamics
$$
d\widehat{X}^{A}_u = \mu(t,\widehat{X}^{A}_u,A_u)  \d u 
+ \sigma(u,\widehat{X}^{A}_u,A_u)\d W_u,
$$
where $W$ is a Brownian motion and $\tilde{\psi}, \tilde{\varphi}$ are functions
independent of randomness.
Euler--Maruyama discretization of this model is a
discrete-time decision problem with
$$
\varphi(X^A,A,Z)= \sum_{t \in \cT}\ \tilde{\psi}(t\Delta t,X^A_t,A_t) \Delta t
+ \tilde{\varphi}(X^A_T),
$$
$$
f(t,X^A_{(t)},A_t,Z_{(t+1)})=X^A_t+ \mu(t\Delta t,X^A_t,A_t) \Delta t
+ \sigma(t\Delta t,X^A_t,A_t) Z_{t+1},
$$
where $A_t=A_{t \Delta t}$, $Z_{t+1}= \Delta W_{t \Delta t}:= W_{(t+1)\Delta t} -W_{t \Delta t}$.
If the original function $\mu$, $\sigma$, $\tilde{\psi}$, $\tilde{\varphi}$
also depend on the randomness in an adapted manner, this would
introduce dependencies of $\varphi$ and $f$ on the past 
Brownian increments as well.
We refer to the classical text book \cite{DK} for 
related results.
}}\end{remark}

Classical Markov decision processes are discussed in Appendix \ref{appendixB}.

\begin{remark}
\label{r.structure}
{\rm{
One important but subtle property of this problem is that, given data, any policy can be evaluated without the need for further data collection or interaction with the system.
For instance, a small investor does not impact stock dynamics when trading, and can therefore observe the returns and reason ex post about what would have happened under other trading strategies.
This is in stark contrast to many engineering applications where controls alter the physical trajectories, and a change in the control requires a new observation or simulation.
This structure, which is pervasive in financial literature, is what allows direct optimization of the strategy (which we formulate as empirical risk minimization in this paper) in lieu of more general reinforcement learning methods.
However, the optimization used in the algorithm requires
large training sets necessitating the
construction of data-driven market models for simulations.
In this exciting new area of research, reinforcement learning 
may play a central role.
}}
\end{remark}

\section{Reformulation}
\label{s.rp}
{In this section, we
{provide a reformulation that}
enables us to write it as a problem of empirical risk minimization in Section~\ref{s.algo}.}
For a feedback action
$a \in \cC$,  an
initial value $x \in \cX$, and  a
(deterministic) trajectory 
$z=(z_1,z_2,\ldots,z_T)$
with $z_0=0$,  we define the
controlled state values 
$\xa=\left(\xa_0,\xa_1,\ldots, \xa_T\right)$
recursively by the equations
$$
\xa_0=x, \quad
\text{and}
\quad
\xa_{t+1}=f(t,\xa_{(t)}, a(t,\xa_t,z_t),z_{(t+1)}),\quad t \in \cT,
$$
where $f$ is as in \reff{e.state},
$\xa_{(t)}=(\xa_0,\ldots,\xa_t)$, $z_{(t+1)}=(z_1,\ldots,z_{t+1})$. The above solution, denoted by 
$\xa(z)=(\xa_1(z),\ldots,\xa_T(z))$,
 is a function of 
the trajectory $z$ and is called the \emph{state function}. Then, for each
process $Z\in \cZ^T$, the 
unique solution of \reff{e.state} is given by 
 $\Xa=\xa(Z)$. 
 Further, let $\alpha^a(z):=(\alpha^a_t(z),\ldots,\alpha^a_{T-1}(z))$ be given 
 by, $\alpha^a_t(z):= a(t,x^a(z),z_t)$ for $t \in \cT$ so that $A^a=\alpha^a(Z)$.
 
 Set
 \begin{equation}
\label{eq.ell}
\ell (a,z):=
\varphi(\xa(z),\alpha^a(z),z),\qquad
a \in \cC, \ z \in \cZ^T,
\end{equation}
where $\varphi$ is as in \reff{e.prob}.  As
$\Xa=\xa(Z)$ and $A^a=\alpha^a(Z)$, the performance function $v(a)$ of \reff{e.prob} is equal to 
$\E[ \ell(a,Z) ]$. Hence, the dynamic decision problem \reff{e.prob}  is 
equivalent to
\begin{equation}\label{e.probl}
 \text{minimize}\ a \in \cC \ \mapsto \ v(a)=\E[ \ell(a,Z) ].
\end{equation}
Its \emph{optimal value} is given by
\begin{equation}
\label{e.vstar}
v^* := \inf_{a \in \cC} v(a)= \inf_{a \in \cC} \E[ \ell(a,Z) ].
\end{equation}
Precisely this 
structure leads to empirical risk minimization and is {quite conducive} to analysis.
 
\subsection{Adapted and Anticipative controls}
\label{r.anticipative} 
This reformulation of the decision problem is not restricted
to feedback actions.  Indeed, let $\cA_{nt}$
 be the set of all $\cA^T$-valued random vectors that are
 $\cF_T$ measurable.  Elements of $\cA_{nt}$ can be parametrized 
 by the set $\cB$ of all Borel measurable functions
 $g=(g_0,\ldots,g_{T-1}) : \cZ^T \to \cA^T$:
 \begin{equation}
 \label{e.an}
 \cA_{nt}=\{ A^{g}\  : \ g\in \cB\},
 \qquad
 \text{where}
 \qquad
A^g_t = g_t(Z), \ \ t \in \cT .
 \end{equation}
Proceeding as above, we construct a state function $\hat{x}^{g}:
\cZ^T \to \cX^{T+1}$ so
that $X^{A^{g}}=\hat{x}^{g}(Z)$
and define $\alpha^g$ similarly. Hence,
the cost function $\varphi(X^{A^g}(Z),\alpha^g(Z),Z)$
is a function of $g$ and  $Z$.
Set
$$
\ell(g,z):= \varphi(\hat{x}^{g}(z),\alpha^g(z),z),\qquad
g\in \cB, \ z \in \cZ^T.
$$
with which another empirical risk minimization problem 
also can be formulated with anticipative controls in $\cA_{nt}$.

Note that
processes in $\cA_{nt}$ are \emph{not adapted} to $\F$
and at any time they may use all available information.  
As they may anticipate and use the future,
we refer to them as \emph{anticipative actions}.
Clearly,  $\cA_{nt}$ is  strictly larger than the set  $\cA_{ad}$
of actions  \emph{adapted} to $\F$ and therefore,
$$
v^*_{nt}:= \inf_{A \in \cA_{nt}}  v(A)
\le v^*_{ad}:= \inf_{A \in \cA_{ad}}  v(A)
\le  v^*:=\inf_{a \in \cC}  v(a).
$$
In all non-trivial control problems,
$v^*_{nt} <v^*_{ad}$ and as discussed in Remark \ref{r.markov} 
above, in Markovian
models we usually have $v^*_{ad}=v^*$.
 
\section{Dynamic Deep Empirical Risk Minimization}
\label{s.algo}
In this section, we outline the approach of  \cite{HE,HEJ}
which can be seen as empirical risk minimization
in view of the above reformulation.

The \emph{training set}
is  a collection of $n$ observations of 
the random process $Z$,
$$
\cL_n=\left\{ Z^{(1)},Z^{(2)}, \ldots, 
Z^{(n)}\right\}\quad
{\text{where}}
\quad Z^{(i)}=(Z^{(i)}_1,Z^{(i)}_2,\ldots,Z^{(i)}_T) \in \cZ^T.
$$
On this set, {in correspondence with \eqref{e.probl}},
the empirical \emph{loss function} for $a\in \cC$ is defined by,
\begin{equation}
\label{eq.L}
L(a;\cL_n):= \frac1n \sum_{i=1}^n  \ell(a,Z^{(i)}).
\end{equation}
As we assume that $Z^{(i)}$ are drawn independently from 
their distribution, $L(a;\cL_n)$
is an approximation of $v(a):= \E[\ell(a,Z))]$.

We consider a sequence of \emph{hypothesis spaces}
$$
\cN_k := \big\{ {h}(\cdot;\theta)\ :\ \theta \in \cO_k \big\},
\qquad k=1,2,\ldots,
$$
where for each parameter $\theta$, 
$ {h}(\cdot;\theta) :  \cT \times \cX\times \cZ \to \cA$
{is} a feedback action.
We assume that
the sequence of parameter sets $\cO_k \subset \R^{d(k)}$ 
are compact subsets with increasing  dimensions $d(k)$
and that $h$ is a continuous function of its variables.
In our numerical experiments,  we  use
an artificial neural network with several hidden layers
as our hypothesis space.
However, for theoretical considerations,
the only requirement we impose 
on  the sequence $\cN_k$ is {that they satisfy the below assumption
which can be seen as the approximation capability or 
being asymptotically ``pointwise'' dense in the set of continuous functions}. 
It is well known that sequences of neural networks
have this property as proved by \cite{Cybenko},
\cite{Hornik}.

\begin{assumption}[{Pointwise Density}]
\label{a.approximate}
We assume that for any bounded
continuous function\\ $\hat{a}: \cT \times \cX \times \cZ \to \cA$,
there exists a sequence {$\{a_k\}_{k \in \N}$
such that   $a_k \in \cN_k$ for each $k \in \N$,
and  $a_k$ converges to $\hat{a}$ pointewise, i.e.,
$\lim_{k \to \infty} a_k(t,x,z)= \hat{a}(t,x,z)$ for every $(t,x,z) \in
\cT \times \cX \times \cZ $.}
\end{assumption}

We also make the following simplifying regularity assumption on the 
coefficients.
\begin{assumption}[{Regularity}]
\label{a.bounded}We assume that 
$f, \varphi$ are uniformly bounded and continuous.  
\end{assumption}

As an immediate consequence of the regularity
assumption, 
there exist a constant $c^*$, 
so that $\ell$ defined in \eqref{eq.ell}
satisfies
$$
|\ell(a,z)| \le c^* , \qquad \forall\ z \in \cZ^T, \ a \in \cC.
$$
Moreover, for any $a_n$ converging pointwise to 
a feedback action $\hat{a}$, $\lim_{n \to \infty} \ell(a_n,z)=\ell(\hat{a},z)$
for every $z \in \cZ$.  In particular,
by dominated convergence, $\lim_{n \to \infty} v(a_n)=v(\hat{a})$.

These assumptions easily imply that
the sequence $\cN_k$
can approximate the optimal value.   
Further convergence results are proved in Section \ref{s.estimate}.
Let $v^*$ be as in \eqref{e.vstar} and set
$$
v^*_k:= \inf_{a \in \cN_k}\ v(a) =  \inf_{\theta \in \cO_k}\ v(h(\cdot, \theta)).
$$
\begin{lemma}
\label{l.approximate}
Suppose that {the above density and regularity
assumptions   
hold}.  Then,
$$
\lim_{k \to \infty}\, v^*_k  = v^*.
$$
\end{lemma}
\begin{proof}
Fix $\eps>0$ and let $a^*_\eps \in \cC$ be an $\eps$-minimizer of $v$: $v(a^*_\eps) \le v^*+ \eps$.
In view of {the density assumption}, there exists a sequence 
 $a_k \in \cN_k$ such that
$a_k$ converges to $a^*_\eps$ pointwise.
Then, by {the regularity assumption}, 
$\limsup_{k \to \infty} v^*_k \le
\lim_{k \to \infty} v(a_k) = v(a^*_\eps) \le v^*+ \eps$.
As the opposite inequality $v^* \le v^*_k$ holds trivially for every $k$, we conclude that
$v^*_k$ converges to $v^*$. 
\end{proof}

\paragraph{Training.}
We fix the training set $\cL_n$ and
the hypothesis space $\cN_k$,  and 
\begin{equation}
\label{e.algo}
{\text{minimize}}\quad
\theta \in \cO_k \ \mapsto \ L( {h}(\cdot;\theta);\cL_n).
\end{equation}
As $L,  {h}$ are continuous and $\cO_k$ is compact,
there exists a minimizer $\theta_{k,n}^* \in \cO_k$.  
Then, the continuous function
$A^*_{k,n}:=  {h}(\cdot;\theta_{k,n}^*)$
is the \emph{trained feedback action}
that could be
constructed by $\cN_k$
using $\cL_n$, and
\begin{equation}
\label{e.vkn}
V^*_{k}(\cL_n):=L(A^*_{k,n};\cL_n)= \inf_{a \in \cN_k} L(a;\cL_n)
\end{equation}
is the  \emph{optimal in-sample performance}
of the hypothesis space $\cN_k$ on the given data $\cL_n$.

The effectiveness of this algorithm depends on
the size of the training set $\cL_n$,
the architecture of the hypothesis space $\cN_k$,
and  on their interactions, and our
 main goal is to study these.
Since numerically one can only
construct an approximation of
the  above minimizer,
details of the approximating
optimization procedure are an essential part
of the algorithm.  In our experiments, we use 
a standard stochastic gradient 
descent variant with early stopping based on
the test-set performance (cf.\ Section~\ref{ss.implementation} for implementation details).
In studies  with simulated data, one does not create an initial training set,
but rather simulate new data for each batch
until training stalls or a stopping rule
is satisfied. 

In the literature several alternatives to
direct empirical risk minimization has been discussed 
and their use might be beneficial.
We refer the reader to a recent paper 
\cite{BSh} and the references therein.
As our main goal
is to analyze the original algorithm, we
do not consider these alternatives in this manuscript.

\section{Examples}
\label{s.examples}

We briefly outline two classes of problems
to clarify the model and the notation.  Further
examples can be found in the forthcomig paper \cite{RSTD}.

\subsection{Merton Problem}
\label{ss.portfolio}

Here we only outline a simple portfolio management 
problem in a financial
market with $d$ many assets.
Although this example does not include many 
important modeling details, it must be clear that 
by appropriately choosing $Z, \cA, X^a$ and the dynamics,
one can cover essentially all
Merton type
utility maximization, portfolio management, and hedging problems studied in the literature.
Also problems with different structures such
as free boundary problems studied
in   \cite{BCJ} and the hedging problems with frictions in \cite{BGTW,BGTWM}
can be included in our framework.

Let $S_t \in \R_+^d$ denote the stock price process
and assume that one-period interest rate $r$ is constant.
The control variable $\pi_t=(\pi_t^1,\ldots,\pi_t^d) \in \R^d$
is the amount of money to be invested in each of the stock.
Classically, it is assumed that $\pi_t^i$ could take any value.
Starting with initial wealth of $x >0$, the 
self-financing wealth dynamics
for the portfolio choice $\pi_t $ are given by
$$
X_{t+1} = X_t + \pi_t \cdot Z_{t+1} + r  (X_t-  \pi_t \cdot {\bf{1}})
= (1+r)X_t +\pi_t \cdot (Z_{t+1} - r {\bf{1}}) \quad
t \in \cT,
$$
where  $X_0=x$,
${\bf{1}}=(1,\ldots,1)\in \R^d$ and the return process $Z$ is given by
$$ 
Z_{t+1}= \frac{S_{t+1}-S_t}{S_t} \in \R^d, \quad t \in \cT.
$$ 
We consider feedback controls $\pi_t=a(t,X_t^a,Z_t)$
and let $X^a$ be
the corresponding wealth process. Then, the classical 
problem
is to maximize
$v(a):= \E[U(X^a_T)]$
with a given utility function $U$.

\begin{example}
\label{ex.utility}
{\rm{In Section \ref{s.numerics}, we numerically study the following stylized example
with an explicit solution
in detail, to illustrate the convergence of the algorithm, potential  overlearning,
 and  the influence of the dimension on them.  
We take  the initial wealth  $X_0=x=0$, $T=2$, and use an exponential utility
$U(x)= 1- e^{-\lambda x}$ where $\lambda >0$ is the risk-aversion
parameter. 

To simplify even further, we assume that
initially one dollar is borrowed and invested uniformly on all stocks.
Then,  $\pi_0=(1/d,\ldots,1/d)$ and $X_1= (Z_1 \cdot {\bf{1}})/d -r$ are uncontrolled,
and the  investment problem is to choose the feedback
portfolio $a(Z_1):=\pi_1(X_1,Z_1)\in \R^d$ 
so as to maximize
$$
v(a)=\E\big[ U(X^a_2)\big]= \E\big[ 1- \exp(-\lambda X^a_2)\big],
$$
where
$X^a_2= (1+r)X_1 + a(Z_1) \cdot (Z_2-r {\bf{1}})$.
A more standard way of
comparing different utility values
is the \emph{certainty equivalent} of a utility value $v<1$, given by
$$
\ceq(v):= \frac{1}{\lambda} \ln(1-v) \quad
\Longleftrightarrow
\quad
v= U(\ceq(v)).
$$

In the numerical experiments,
to reduce the output noise,
we fix a unit vector $\eta \in \R^d$
and take $Z_2=\zeta \eta$, 
where the real-valued 
 Gaussian random variable $\zeta$ is
independent of $Z_1$
and has mean $m$
and volatility $s$.
Then, with $r=0$,
$$
a^*(z)=a^*=\frac{m}{\lambda s^2}\ \eta, \quad
v^*=v(a^*)= 1-\exp(-\frac{m^2}{2 s^2}),
\quad
\ceq(v^*)=- a^*.
$$}}
\end{example}

\subsection{Production Planning}
\label{ss.pp}

The multi-stage optimization problems 
introduced by \cite{BSS,BSS2}
is in the above structure as well.
Although not in
quantitative finance, 
here we describe a simple
example of these problems
that is very similar to Example 1 in \cite{BSS},
to clarify many of the notions introduced in the paper.

We consider producers facing 
 an optimal
production decision. 
They observe the random demand $Z_1,Z_2$
in two stages. 
The production level $a(Z_1)$ is decided
after  observing $Z_1$ but before $Z_2$
and the second component of the random 
demand $Z_2$ is observed afterwards, at the  final stage.
The goal is to bring 
the final inventory level close to zero
 by properly choosing the production
level at stage one.  Let $\Xa$ be the inventory level.
We assume the initial
inventory is zero and
no production is made initially.
Then, 
$\Xa_1=Z_1$,
$\Xa_2=\Xa_1- a(Z_1)+Z_2$,
 and the problem is to minimize
$$
v(a)=\E\big[ \varphi(\Xa_2)\big]
=\E\big[ \varphi\big(Z_1+Z_2-a(Z_1)\big)\big]
$$
over all production functions $a$. 
The penalty function $\varphi \ge 0$ is convex and is 
equal to zero only at the origin. 
In our framework,   $\cA=[0,\infty)$, and
 $f(t,x,z,a)= y-a+z$.

For $\varphi(x)=x^2$
this is exactly the classical regression problem 
of estimating the total demand $Z_1+Z_2$
after observing the first component $Z_1$.  
 It is well-known that the 
 optimal solution is
$a^*(Z_1) = \E[Z_1 +Z_2\ | \ Z_1]$,
and this optimization
problem  reduces to the 
classical regression
 well-known to face
the bias-variance trade-off.
Although this connection may not be as explicit
in other more complex models,
it is always inherent to the problem.

\section{Overlearning}
\label{s.overlearning}

Recall the set of anticipative controls $\cA_{nt}$ of Section~\ref{r.anticipative},
{$V^*_k(\cL_n)$ of \eqref{e.vkn} and set
$$
V^*(\cL_n):=  \lim_{k \to \infty}\, V^*_k(\cL_n)=
\lim_{k \to \infty} \,  \inf_{a \in \cN_k}\ L(a;\cL_n).
$$
}The following result is proved
in Section~\ref{ss.asy} below, under the natural
assumption of distinct data
and its relaxation  is discussed in the Appendix \ref{appendixA}.

\begin{theorem}
\label{t.ant}
Suppose that {the density and the regularity assumptions  {\rm{(}}c.f.~Assumptions
 \ref{a.approximate}, \ref{a.bounded}{\rm{)}}} hold and
 the training data is distinct, i.e., 
for every $t \in \cT$ and $i\neq j$, $Z_t^{(i)} \neq Z_t^{(j)}$. Then,
{$$
\limsup_{n \to \infty} V^*(\cL_n) \le 
v^*_{nt}:=\inf_{A \in \cA_{nt}}\  v(A).
$$}
\end{theorem}

In all non-trivial decision problems, $v^*_{nt} < v^*:=\inf_{a \in \cC} v(a)$. 
Thus, sufficiently large hypothesis spaces, 
 in-sample,
overperform the optimal value $v^*$.
In optimal control, it is centrally important that
the decisions are adapted to the information  flow.
The above results show  that 
hypothesis spaces 
circumvent this restriction by 
predicting the future values the data and are thus able to overperform
on the training set. 
{We emphasize that the training is 
done in the hypothesis class
$\cN_k$ and the elements of $\cN_k$ are
of feedback form.  Hence, they are naturally adapted processes.
This is further discussed in Remark \ref{r.nonadapted} below.}
We refer to this possibility  as \emph{overlearning}.

\subsection{Examples}
\label{ss.examples}

We return to two examples from Section \ref{s.examples}  to clarify the above discussion. 

We first consider the production planning
problem of Section \ref{ss.pp}.
The only feedback action in that context is
the production decision.  
For a fixed control $\alpha \in \R$
and given demands $z=(z_1,z_2)$,
the cost function
$\ell(\alpha,z)= 
\varphi(z_1+z_2-\alpha) \ge 0$
is zero at the origin. Hence,
$a^*(z)=z_1+z_2$ is the pointwise optimizer.
If the training data $\cL=\{Z^{(1)},Z^{(2)}, \ldots, 
Z^{(n)}\}$ is distinct, 
any sufficiently large $\cN_k$
has an element  $a^*_k$
such that $a^*_k(Z_1^{(i)})$ is uniformly close to $Z_2^{(i)}+Z_2^{(i)}$ for
each $i$.  Then, the feedback action
$a^*_k$
constructed by $\cN_k$ achieves an in-sample performance value
 close to zero yielding $V^*(\cL_n)=0$.
As $v^*>0$, 
this would be overlearning
and $a^*_k$ does not generalize.

Next, consider the utility maximization problem discussed in Example 
\ref{ex.utility} with one stock. 
Then, for control $\alpha \in \R$
and  returns $z=(z_1,z_2)$,
$\ell(\alpha,z)=  1- \exp(-\lambda[(1+r)(z_1-r)+\alpha(z_2-r)])$.
By taking arbitrarily large positions
depending on the sign of $z_2-r$, 
one obtains $V^*(\cL_n)=1$.  In financial terms,
large enough hypothesis spaces anticipate the sign of the random variable $Z_2-r$ 
by observing $Z_1$ and use it to create numerical arbitrage
caused by the obvious non-adaptedness and overlearning
on the training data.  In particular,
the trained feedback actions almost
achieve a performance value of one,
and thus overperform the optimal value $v^*<1$.
Additionally,
in this example, the optimal value
obtained by the anticipative controls
is also equal to one, $v^*_{nt}=1$,
which is consistent with Theorem \ref{t.ant}.

\begin{remark}
\label{r.pp}{\rm{In closely related
studies Pflug and Pichler analyze 
the dependence of optimization
problems on the distribution
of the randomness.  In our terminology,
they prove that
overlearning implies that the 
limit of the values obtained by empirical
measures do not convergence.  Motivated
by this observation, they carefully define the nested (or adapted)
distance among probability measure which
yields the continuity of the value function,
cf.~Proposition 1 in \cite{PP}.  A similar
observation is also made in Example 7.1 in \cite{BBB}.
}}

\end{remark}

\subsection{Asymptotic Overlearning}
\label{ss.asy}
We continue with an estimate used in the proof of Theorem \ref{t.ant}.
For any $\alpha \in \cA^T$, we define a constant (in space) action by
$A^\alpha_t:=\alpha_t$.  With abuse of notation, we consider $\alpha\in \cA^T$
as an element of $\cC$. 
Recall that the  anticipative actions $\cA_{nt}$  are parametrized by $\cB$  of $\cF_T$ measurable functions, cf.~\reff{e.an}.

\begin{lemma}
\label{l.equal}
It holds that, 
\begin{equation}
\label{e.vn}
L^*_n:=\frac1n \sum_{i=1}^n \inf_{\alpha \in \cA^T} \ell(\alpha,Z^{(i)}) 
\le\  \frac1n \sum_{i=1}^n  \ell(g,Z^{(i)})=L(A^g;\cL_n),
\qquad \forall g \in \cB.
\end{equation}
\end{lemma}
\begin{proof}
Fix $\hat{z} \in \cZ^T$, $g=(g_0,\ldots,g_{T-1})\in \cB$ and
let $\hat{x}^{g}$ be the state function defined in Section~\ref{r.anticipative}.
Define a constant action 
$\hat{\alpha}:=(\hat{\alpha}_0,\ldots,\hat{\alpha}_{T-1})$
by $\hat{\alpha}_t= g_t(\hat{z})$ for $t \in \cT$.
Let $x^{\hat{\alpha}}$
be the corresponding state process.
Then, by a simple induction argument on the time variable,
we can show that 
$x^{\hat{\alpha}}(\hat{z})= \hat{x}^{g}(\hat{z})$
(the process $x^{\hat{\alpha}}(z)$ is possibly
not equal to $\hat{x}^{g}(z)$ for trajectories $z$ other than $\hat{z}$).
Therefore, 
$\ell(\hat{\alpha},\hat{z})=\ell(g,\hat{z})$ and
$$
\inf_{\alpha \in \cA^T} \ell(\alpha,\hat{z}) 
\le \ell(\hat{\alpha},\hat{z})=\ell(g,\hat{z})
\quad
\Rightarrow
\quad
\inf_{\alpha \in \cA^T} \ell(\alpha,Z^{(i)}) 
\le  \ell(g,Z^{(i)}), \quad i=1,\ldots,n.
$$
The inequality $L^*_n \le L(A^g;\cL_n)$ now
follows  directly.
\end{proof}

\begin{theorem}
\label{t.learning}
Under the hypotheses of Theorem \ref{t.ant},
$V^*(\cL_n)= L^*_n$ for every $n$.
\end{theorem}

\begin{proof} 
Fix $ \cL_n, \eps>0$ and for $i=1,\ldots,n$,  choose  $\alpha^{(i)} \in \cA^T$ satisfying 
$$
\ell({\alpha^{(i)}},Z^{(i)}) \le \inf_{\alpha \in \cA^T} \ell(\alpha,Z^{(i)}) + \eps.
$$
Since $Z^{(i)}$ are distinct, there exists  a bounded, smooth function
$a_\eps : \cT \times \cX\times \cZ \to \cA$,
such that
$$
a_\eps(t,\xi,\Zi_t)=
\alpha^{(i)}_t, \quad t \in \cT, \ \xi \in \cX^T, \ \ i=1,\ldots,n.
$$
For each $\Zi \in \cL_n$, by induction over time, one can show that
$x^{a_\eps}(\Zi)=x^{{\alpha^{(i)}}}(\Zi)$.  Therefore,
$\ell(a_\eps,\Zi)=\ell({\alpha^{(i)}},Z^{(i)})$ for each $i$, and

$$
L(a_\eps;\cL_n)=\frac1n \sum_{i=1}^n
 \ell(a_\eps,Z^{(i)})= \frac1n \sum_{i=1}^n
 \ell({\alpha^{(i)}},Z^{(i)})
 \le \frac1n \sum_{i=1}^n
  \inf_{\alpha \in \cA^T} \ell(\alpha,Z^{(i)}) +\eps
  =L^*_n+\eps.
 $$

Moreover, by the {density assumption (c.f.~Assumption \ref{a.approximate}), 
there is a sequence  $\{a_k \in \cN_k\}_{k \in \N}$}  (depending on the fixed training set $\cL_n$
and $\eps$)  that
approximates $a_\eps$ pointwise.
We now use {the regularity assumption (c.f.~Assumption  \ref{a.bounded})}
to conclude that
$$
\lim_{k\to \infty}\, L(a_k;\cL_n)= L(a^{\eps};\cL_n).
$$
Hence,
$$
V^*(\cL_n) 
:= \lim_{k \to \infty} \  \inf_{a \in \cN_k}\ L(a;\cL_n) 
 \le \lim_{k\to \infty} L(a_k;\cL_n)
 =  L(a^{\eps};\cL_n)
  \le   L^*_n +\eps,
$$
and consequently, $ V^*(\cL_n) \le L^*_n $.
The opposite inequality follows from Lemma \ref{l.equal}.  Indeed, 
as $\cN_k \subset \cC \subset \cB$,
 \reff{e.vn}  implies that
 $L^*_n \le L(a;\cL_n)$
 for any $k$ and $a \in \cN_k$.  As $L^*_n$
 is independent of $k$ and $a$,
 we first take the infimum over $a \in \cN_k$ and 
 then let $k$ tend to infinity to arrive at
 $L^*_n \le V^*(\cL_n)$.
\end{proof}

\begin{proof}(of Theorem \ref{t.ant}).
In view of  Theorem  \ref{t.learning}
and Lemma \ref{l.equal}, 
$V^*(\cL_n)=L^*_n \le L(A^g;\cL_n)$ 
for any $g \in \cB$. 
As the training data is drawn independently from
the distribution of $Z$, by  law of large numbers, 
$$
\limsup_{n \to \infty}  V^*(\cL_n)\le \lim_{n \to \infty} L(A^g;\cL_n)
=\lim_{n \to \infty} \frac1n \sum_{i=1}^n \ell(g,Z^{(i)}) =  \E[\ell(g,Z)]
=v(A^g).
$$
We complete the proof by taking infimum over $g\in \cB$.
\end{proof}

\begin{remark}
\label{r.extend}
{\rm{If for every $\eps>0$, there exists $g_\eps \in \cB$ satisfying
$$
\ell(g_\eps(z),z) \le \inf_{g \in\cB} \ell(g,z)+\eps, \qquad \forall z \in \cZ,
$$
then, instantly it follows that $\limsup_{n \to \infty} V^*(\cL_n)=v^*_{nt}$.   Moreover,
one may construct $g_\eps$  through a 
standard use of a measurable selection theorem
under
 some mild additional assumptions on the functions.
As this result is tangential to the main trust of the paper,
we chose to omit this technical discussion.}} 
\end{remark}

\begin{remark}
\label{r.nonadapted}
{\rm{We have shown that
the trained actions may overperform the optimal value $v^*$.  
However, as they
are in feedback form, theoretically the expected value of
their performance is bounded by $v^*$.
So overperformance is a subtle and a data-dependent one.
Indeed, the \emph{coefficients} of the trained
actions use the future data explicitly
and therefore become \emph{non-adapted}
on the \emph{training data} and 
the upper bound $v^*$ obtained by adapted actions
does not hold.
On the other hand, their out-of-sample performance
are bounded by $v^*$ and in our numerical studies they
underperform substantially.

The proof 
of Theorem \ref{t.learning}  
also shows the importance of the dimension $d$ as well.  Indeed, 
in higher dimensions,
the training data is `more and more distinct'  
allowing for easier overlearning, an effect we observe numerically as well.
The separation between the training data is also a factor in 
the Rademacher complexity that is discussed in the next section.
}}
\end{remark}

\begin{remark}[Regularization]{\rm{
\label{r.regular}
A common approach to reduce over-learning
is to add regularization such as
restricting the hypothesis classes $\cN_k$
to be subsets of  the set of $K$-Lipschitz
functions,
$$
Lip_K:= \{ h \in \cC\ :\
|h(z)-h(z')|\le K|z-z'|\}.
$$
Then, we argue in Remark \ref{rem.lip} below
that Rademacher complexity of 
these restricted spaces goes to zero 
as the training data gets larger.
As we prove in the subsection \ref{ss.convergence}
below, this convergence implies that
$$
\lim_{n \to \infty}\ \lim_{k\to \infty} \ \inf_{a \in \cN_k 
\cap Lip_K} L(a;\cL_n) = \inf_{a \in \cC 
\cap Lip_K} v(a).
$$
In many control problems, the right-hand side of above 
converges to the optimal value as the Lipschitz constant $K$ gets larger.  
However, in financial applications this constant is large and thus, we are 
close to the limit considered Theorem \ref{t.learning},
making the data need very large.
Another numerical difficulty is to restrict the Lipschitz constant of deep neural networks.
}}\end{remark}

\section{Estimates and Convergence}
 \label{s.estimate}
We first recall several classical definitions and results, 
{cf.~\cite{BM,BBL4,BBL3,BBL2,BBL,KP}.}

Let $\cG$ be a hypothesis space
of  a set of real-valued functions defined  on the set of
trajectories. 

\begin{definition}
\label{d.er}{\rm{The}}
empirical Rademacher complexity {\rm{of
$\cG$ on the training set $\cL_n$ is given by
$$
R_e(\cG;\cL_n):= \E\Bigg[
\sup_{g \in \cG} \ \frac 1n \sum_{i=1}^n \sigma_i g(Z^{(i)})\,
{\big{|}\, \cL_n\, }\Bigg],
$$
where the expectation is over the}}
Rademacher variables {\rm{$\sigma_i$, 
which are identically and independently distributed 
taking values $\pm1$ with equal probability.}}
\end{definition}

\begin{definition}
\label{d.rade}{\rm{The}}
Rademacher complexity {\rm{of
$\cG$  is given by
$$
r(\cG;n):= \E[ R_e(\cG,\cL_n)],
$$
where the expectation is over the 
random training set $\cL_n=\{Z^{(1)},Z^{(2)},\ldots,Z^{(n)} \}$
whose elements 
are independently and identically drawn
and the dependence on this distribution
is not shown in our notation.
}}
\end{definition}

Let $v:\cC \to \R$ be as in \eqref{e.prob}, $L$ be 
as in \eqref{eq.L}. For $\cN_k$, $\cL_n$, set
$$
G(\cN_k,\cL_n):
=\sup_{a \in \cN_k}\ | 
v(a)
- L(a;\cL_n)|.
$$
The following result that uniformly connects empirical averages to 
expected values is classical.
Suppose that $\cL_n$ is drawn independently and identically
and $|g| \le c^*$ for every $g \in \cG$.
Then, for a given 
$\delta\in (0,1)$, 
with probability at least $1-\delta$
the following 
estimates hold,
\begin{equation}
\label{e.uniformestimate}
G(\cN_k,\cL_n)
\le c(\cN_k,n,\delta)\le C_e(\cN_k,\cL_n,\delta),
\end{equation}
where with $\ell(\cN_k)= \{ \ell(a, \cdot)\  | \ a \in \cN_k\}$,
\begin{align*}
c(\cN_k,n,\delta)&:= 2 r(\ell(\cN_k); n) + 2 c^* \sqrt{\frac{\ln(2/\delta)}{2n}},\\
C_e(\cN_k,\cL_n,\delta)&:= 2 R_e(\ell(\cN_k);\cL_n) + 6 c^* \sqrt{\frac{\ln(2/\delta)}{n}}.
\end{align*}
One-sided version of these estimates for functions $0\le g \le 1$
is proved, for instance in Theorem 3.3 by \cite{MRT}
and elementary arguments yield the above two-sided estimates.

\begin{remark}[Complexities of Neural Networks]
\label{rem.monotone}
{\rm{
Suppose that the hypothesis spaces $\cN_k$
is a sequence of neural networks with increasing depth and width.

As  the neural networks $\cN_k$ get wider and deeper,
the constant $c(\cN_k,n,\delta)$
and the random variables $G(\cN_k,\cL_n), C_e(\cN_k,\cL_n,\delta)$
increase.
The monotonicity in the $\delta$ variable is also clear.  
One may obtain
further estimates by using the Rademacher calculus as described
in Section 26.1 of \cite{SSB}.
Indeed, if the mapping
$a \in \cC \mapsto \ell(a,z) $ is uniformly Lipschitz,
then the Kakade \& Tewari composition Lemma
(see \cite{KT},  also Lemma 26.9 in \cite{SSB}) implies that
one can estimate the complexities $R_e(\ell(\cN_k);\cL_n)$ and $r(\ell(\cN_k);n)$
by the Rademacher complexities
$R_e(\cN_k;\cL_n)$,
$r(\cN_k;n)$ of the neural networks. 

Moreover, {one may 
use the Massart Lemma together
with what is known as the Pisier's trick and the composition lemma,
to prove that} the Rademacher complexity $r(\cN_k;n)$ of the neural networks
converges to zero
as the size $n$ of the training data goes to infinity; see for example
problem 3.11 in \cite{MRT}
or Corollary 3.8 in the lecture notes of \cite{Wolf}.
In fact detailed estimates are also available in \cite{GRS,NTS}.
Since the regularity of $\ell$ can be directly
proven under Lipschitz assumptions on the
coefficients of the decision problem,
this procedure shows that under natural assumptions
on the coefficients,
 the complexity  $r(\ell(\cN_k);n)$
also converges to zero.
}}
\end{remark}

\begin{remark}[Lipschitz spaces]
\label{rem.lip}
{\rm{The complexities
$R_\eps(\ell(\cN_k \cap Lip_K);\cL_n)$
can be effectively estimated by the
deep convergence rates obtained in \cite{FG}
for the uniform
convergence of the empirical measure
in the Wasserstein  metric $W_1$.  As
for any two probability measures $\mu, \nu$,
and $a \in Lip_K$
$$
|(\mu, a) -(\nu,a)| \le K W_1(\mu,\nu),
$$
convergence of the complexity directly  follows
from \cite{FG}.}}
\end{remark}

\subsection{Complexity Estimates}
\label{ss.estimate}

Recall that $v^*$ is defined in \eqref{e.vstar},
$\theta^*_{k,n}$ is {a} minimizer of \reff{e.algo} and
$A^*_{k,n} := {h}(\cdot;\theta^*_{k,n})$ is an optimal
feedback action that can be constructed by $\cN_k$ on the
set $\cL_n$.
Let $\widehat{\cL}_n$ be another data set drawn identically
and independently from the same distribution as $\cL_n$. 
In this section, we obtain empirical bounds on the differences
of the in-sample performance $L(A^*_{k,n};\cL_n)$,
out-of-sample performance $L(A^*_{k,n};\widehat{\cL}_n)$,
and the average performance $v(A^*_{k,n})$
of $A^*_{k,n}$,
as well as their deviations from $v^*$.

\begin{theorem}
\label{t.rademacher}
Under {the density and the regularity assumptions  {\rm{(}}c.f.~Assumptions
 \ref{a.approximate}, \ref{a.bounded}{\rm{)}}} ,
for every $\eps>0$,
there exists $k_\eps$  such that
\begin{equation}
\label{e.rate1}
\left|v^*- L(A^*_{k,n};\cL_n)\right| \le G(\cN_k,\cL_n)+ \eps, 
 \quad \forall \ k \ge k_\eps.
\end{equation}
In particular, for all $\delta>0$ 
the following hold with at least $1-\delta$
probability for every $k \ge k_\eps$,
\begin{align}
\label{e.rate}
\left|v^*- L(A^*_{k,n};\cL_n)\right|  \le 
c(\cN_k,n,\delta) + \eps \le C_e(\cN_k,\cL_n,\delta)+ \eps, \\
 \nonumber
 \left|v^*- v(A^*_{k,n})\right|  \le 
 2c(\cN_k,n,\delta) + \eps \le 2 C_e(\cN_k,\cL_n,\delta)+ \eps.
\end{align}
\end{theorem}
\begin{proof}
For $\eps>0$ choose $a_\eps^* \in \cC$ satisfying
$v(a_\eps^*) \le v^* +\frac12 \eps$.
By
Assumptions \ref{a.approximate} and \ref{a.bounded}, 
there exists a sequence $a_k \in \cN_k$
and $k_\eps $ such that 
$v(a_k)\le v(a_\eps^*)+ \frac12 \eps $
for all $k \ge k_\eps$.  Hence,
$$
v(a_k)\le v(a_\eps^*)+ \frac12 \eps \le v^* + \eps,
\quad \forall k \ge k_\eps.
$$
Since $ L(A^*_{k,n};\cL_n) \le L(a;\cL_n)$ for any $a \in \cN_k$,
the definition of $G$ implies that
\begin{align*}
 L(A^*_{k,n};\cL_n)& \le  L(a_k;\cL_n) 
 \le v(a_k)+ G(\cN_k,\cL_n)\\
& \le v^*+G(\cN_k,\cL_n)+ \eps,
 \quad \forall k \ge k_\eps.
 \end{align*}
 As  $ v^*:= \inf_{a \in \cC} v(a)$ and 
 $A^*_{k,n} \in \cN_k \subset \cC$,
$v^* \le  v(A^*_{k,n}) \le L(A^*_{k,n};\cL_n)+G(\cN_k,\cL_n)$.
Now \eqref{e.rate1} follows 
 from the above inequalities, and 
 \eqref{e.rate} follows 
 from \eqref{e.rate1} and \eqref{e.uniformestimate}.
 Finally,
 $$
  \left|v^*- v(A^*_{k,n})\right|  \le  \left|v^*- L(A^*_{k,n};\cL_n)\right| 
  +\left| v(A^*_{k,n})- L(A^*_{k,n};\cL_n)\right|  \le 
 2 G(\cN_k,\cL_n)+ \eps.
$$ 
\end{proof}

By  \eqref{e.uniformestimate},
the following
holds with at least $1-\delta$ probability
for every $a \in \cN_k$,
\begin{align*}
\left| L(a; \widehat{\cL}_n)-  L(a;\cL_n)\right|  &\le
\left| v(a)-  L(a;\cL_n)\right| +\left| v(a)- L(a; \widehat{\cL}_n)\right|\ \\
 &\le G(\cN_k,\cL_n) + G(\cN_k,\widehat{\cL}_n)
\le 2 c(\cN_k,n,\delta/2).
\end{align*}
This shows that the in-sample  and out-of-sample performance difference
of any network provides an empirical lower bound for the 
Rademacher complexity with high probability.
In fact, in many applications it is a standard practice to monitor this difference.
Thus, also in view of the estimate \eqref{e.rate},
the following quantity,
\emph{maximal performance difference}, could be taken as a
proxy for overlearning,
 $$
O(\cN_k,\cL_n,\widehat{\cL}_n):=
\sup_{a \in \cN_k} |L(a;\cL_n) - L(a;\widehat{\cL}_n)| .
$$
We restate that the following
holds with at least $1-\delta$ probability,
 \begin{equation}
\label{e.diff}
 |L(a;\cL_n) - L(a;\widehat{\cL}_n)|\le
 O(\cN_k,\cL_n,\widehat{\cL}_n) \le 
2 c(\cN_k,n,\delta/2),
\quad
\forall a \in \cN_k.
\end{equation}

\subsection{Convergence}
\label{ss.convergence}

In this subsection, we prove convergence under
an assumption on the complexity of the hypothesis spaces $\cN_k$.
Recall $A^*_{k,n}
=h(\cdot,\theta^*_{k,n})$  of \eqref{e.vkn}.

\begin{corollary}
\label{c.convergence}
{Suppose that for each $k$,
the Rademacher complexity $r(\ell(\cN_k);n)$
converges to zero as the training size $n$ tends to zero. Then,} 
the following
holds  with probability one,
$$
\lim_{ n \to \infty} L(A^*_{k,n};\cL_n)  = 
\lim_{ n \to \infty} L(A^*_{k,n};\widehat{\cL}_n)  = 
\lim_{ n \to \infty} v(A^*_{k,n})  =v^*_k:= \inf_{a \in \cN_k} v(a).
$$
\end{corollary}

As discussed in Remark \ref{rem.monotone} above, the above
assumption on the complexity is satisfied by 
the neural networks under natural assumptions 
on the coefficients.  
Also,  in  Lemma \ref{l.approximate}, 
{under the density and the regularity assumptions,}
we have shown that $v^*_k$ converges
to $v^*$.
Hence, the above result states that as the size of training data increases,
the performance of the feedback actions constructed by $\cN_k$
converge to the optimal value provided that the size of the hypothesis
classes also
tends to infinity in a controlled manner.

\begin{proof} 
By the definition of $G$,
$v^*_k \le v(A^*_{k,n}) \le L(A^*_{k,n};\cL_n) + G(\cN_k,\cL_n)$.
Also, for any $a \in \cN_k$,
$$
L(A^*_{k,n};\cL_n) \le L(a;\cL_n) \le v(a)+ G(\cN_k,\cL_n).
$$
We take the infimum over $a \in \cN_k$ to conclude that
$L(A^*_{k,n};\cL_n) \le v^*_k+ G(\cN_k,\cL_n)$. Hence,
$$
\left | v^*_k -L(A^*_{k,n};\cL_n)
\right| \le G(\cN_k,\cL_n).
$$

Fix $\eps>0$ and set $\delta_{n}= 2 \exp(-2 n \eps^2/(6c^*)^2)$ so 
that 
$$
c(\cN_{k},n,\delta_n)=2 r(\ell(\cN_k); n) + 6 c^* \sqrt{\frac{\ln(2/\delta_n)}{2n}}
= 2 r(\ell(\cN_k),n) + \eps.
$$  
Then, by  \eqref{e.uniformestimate},
for every $k$
with at least $1-\delta_{n}$ probability
$$
\left|v^*_k-L(A^*_{k,n};\cL_n) \right|
 \le G(\cN_k,\cL_n)\le
c(\cN_{k},\cL_n,\delta_n) = 2 r(\ell(\cN_k),n) + \eps.
$$
Equivalently, $\P(\Omega_{k,n,\eps}) \le \delta_{n}$, where
$$
\Omega_{k,n,\eps}:=
\{ |v^*_k-L(A^*_{k,n};\cL_n) | > 2 r(\ell(\cN_k),n) + \eps\}.
$$
Since $\sum_n \delta_{n}<\infty$, by the
Borel--Cantelli Lemma, for every $k$,
$$
\limsup_{ n \to \infty} \left|v^*_k-L(A^*_{k,n};\cL_n) \right| \le 
\lim_{n \to \infty} 2 r(\ell(\cN_k),n) +  \eps = \eps,
$$
with probability one.

In view of  \eqref{e.diff},
by at least $1-\delta_n$ probability
$$
\left| L(A^*_{k,n}; \widehat{\cL}_n)-  L(A^*_{k,n}; \cL_n)\right| \le 
2 c(\cN_k,n,\delta_n/2)
= 4  r(\ell(\cN_k),n) +  \ln(4) \eps.
$$
The above Borel--Cantelli argument
also implies that with probability one,
$$
\limsup_{ n \to \infty}
\left| L(A^*_{k,n}; \widehat{\cL}_n)-  L(A^*_{k,n}; \cL_n)\right| \le \ln(4) \eps.
$$
\end{proof}

\section{Numerical Experiments}
\label{s.numerics}
In this section we present
the numerical implementations of Example~\ref{ex.utility}.
We take $\lambda=1, r=0$ and as discussed in that example
the return of the second period 
$Z_2=\zeta \eta$, where  $\zeta$
is Gaussian with mean $18\%$ and volatility $0.44\%$\footnote{We
have chosen the mean and the volatility values randomly among those
with $a^*$ close to one and which are neither too small or large.
For these parameter values, overlearning is not
particularly easy.}
and is independent of $Z_1$.
Then, the optimal solution given in Example~\ref{ex.utility}
is $a^*=0.9297$ and $\ceq(v^*)=-a^*$.
Moreover, for all parameters,
$v^*_{nt}=1$ yielding $\ceq(v^*_{nt})=-\infty$.
To focus the training on a compact input domain, 
$Z_1\in \R^d$ is distributed uniformly over 
the $d$-dimensional hypercube $[-0.5, 0.5]^d$.
We use the certainty equivalent $\ceq$ defined in Example~\ref{ex.utility}
to compare the
performance of different actions.

As our main goal is 
to  illustrate  the potential overlearning,
we try to strike a balance between avoiding unnecessary 
tuning parameters while still implementing commonly accepted best practices.
The simple but representative  structure of the chosen example allows
us to easily evaluate the
trained feedback actions by comparing them to explicit formulae,
and also provides an understanding of the performance of
this algorithm on a general class of decision problems. 
We emphasize that our claim is not that overlearning cannot 
be alleviated in these problems,
but that it does occur even with a seemingly reasonable learning setup
and that one has to be aware of the possibility. 
Indeed some degree of tuning could possibly lead to improvement
in this particular example,
but such methods are not systematic, and it is 
not clear that they generalize when the ground truth is not available.
Corollary~\ref{c.convergence} and Lemma~\ref{l.approximate} show that increasing the training set
(and possibly the architecture complexity in a controlled manner)
does provide a systematic method for improvement.
This is also observed in the computations that follow.

To describe our findings succinctly,
let $a^*$ be the (constant) optimal feedback action
and $A_{k,n}$ be the feedback action computed 
by the neural network $\cN_k$ on the training set $\cL_n$.
Although the optimization algorithm is trying
to compute the minimizer $A^*_{k,n}$ of \eqref{e.vkn},
in actual computations, the stochastic gradient algorithm 
is stopped before reaching $A^*_{k,n}$.
Thus, $A_{k,n}$ depends not only 
on the training data $\cL_n$ and the network $\cN_k$
but also on the 
optimization procedure, in particular, the stopping rule.

By taking advantage of the explicitly 
available solution, we define the in-sample relative performance
$p_{in}$ and the out-of-sample
relative performance $p_{out}$ of
the trained actions $A_{k,n}$
by,
$$
p_{in}:= 
 \frac{\text{nn}_\text{in-sample} - \text{true}_\text{in-sample}}{\text{true}_\text{in-sample}},
 $$
 $$
 p_{out} := 
  \frac{\text{nn}_\text{out-of-sample} 
  - \text{true}_\text{out-of-sample}}{\text{true}_\text{out-of-sample}},
 $$
where 
\begin{align*}
&\text{nn}_\text{in-sample} := \ceq(L(A_{k,n};\cL_n)),\quad
&\text{nn}_\text{out-of-sample} := \ceq(L(A_{k,n};\widehat{\cL}_n)),\\
&\text{true}_\text{in-sample} := \ceq(L(a^*;\cL_n)),\quad
&\text{true}_\text{out-of-sample} := \ceq(L(a^*;\widehat{\cL}_n)),
\end{align*}
and $\cL_n$ is the training set used to compute $A_{k,n}$ and
$\widehat{\cL}_n$ is the training set chosen identically and  independently of 
$\cL_n$.  Then, the appropriately normalized  performance difference $p_{in}-p_{out}$
provides an understanding of  the overlearning proxy 
$O(\cN_k,\cL_n,\widehat{\cL}_n)$ as in \eqref{e.diff}.
Indeed, larger values of the difference imply
larger values of $O$.

We focus on these measures, $p_{in}, p_{out}$, for two reasons.
Firstly, although our samples are large enough to give a 
good representation of the distribution,
the above formulae eliminate
some dependency on the sample by subtracting the true 
optimizers performance on each sample.
Secondly,  there are circumstances where 
seemingly the training immediately tries 
to interpolate data instead of first approaching 
the true solution before starting to interpolate, as one might expect.
This leads the out-of-sample performance to increase very early on, and
with our stopping rule based on the out-of-sample performance, 
it thus leads to almost immediate stopping.
In this sense, the combination of stopping rule and 
performance measure is relatively conservative for
measuring overlearning.

\subsection{Implementation Details}\label{ss.implementation}
Our implementation is written in the programming language OCaml 
\citep{ocaml} using the library Owl \citep{owl}.
To fully reproduce the computations, the code and the logs, including 
the random number generator seeds,  
are available at 
\url{https://gitlab.com/mreppen/dderm}.
Note that most neural network `best practices' are developed 
for other types of problems than the control problems we study here.
We still follow these practices along with common defaults so as not
to  color the results by specific choices.
Nevertheless, we observe the same qualitative results also with alternative implementations.
In all examples, the activation functions are set to ReLU 
and the parameters are optimized by stochastic gradient 
descent using the Adam scheme with parameters 
$(\alpha, \beta_1, \beta_2) = (0.001, 0.9, 0.999)$, as 
proposed by \cite{kingma2014adam}.
The neural networks are constructed with three hidden layers.
This architecture is kept fixed regardless of data dimensionality 
 to better isolate the dimensionality's impact on overlearning.
The weights are initialized with a uniform centered distribution of 
width inversely proportional to the square root of the number 
of neuron inputs.\footnote{The uniform He-initializer 
\cite{He_2015_ICCV}---which differs only by a factor 
$\sqrt{6}$ in the width of the uniform distribution and is 
commonly recommended for training ReLU 
networks---has not shown qualitatively different results 
with regards to overlearning.}

As the neural networks are capable of overlearning the data,
we must employ stopping rules for early stopping.
Such stopping rules are commonly used in practice as implicit regularizers.
In our studies, we mainly use a conservative stopping rule that monitors 
the out-of-sample performance on a separate validation set after each epoch and terminates 
when that out-of-sample performance exceeds its past 
minimum.\footnote{As the parameter landscape is expected to 
have plateaus and the out-of-sample performance is not 
expected to be perfectly monotone, this calls for some 
tolerance, thereby introducing a tuning parameter.
To be conservative, we keep this tolerance small to 
encourage early stopping and reduce overlearning.}
To demonstrate the potential overlearning,
we also performed some experiments
running the stochastic gradient without stopping for a fixed number of epochs.

We train using minibatches sampled randomly from the training set.
Overlearning can also be observed with batch gradient descent---equivalent to the extreme case of setting the minibatch size to the full training set---but we have 
opted to default to minibatches as it is far more common and computationally 
efficient\footnote{The computational burden of each gradient computation scales as $\bigO(N)$ in the batch size $N$, but the accuracy is of order $\bigO(1/\sqrt{N})$, leading to computational advantages of small batch sizes (but not too small, due to SIMD instructions in modern CPUs and GPUs).
It is sometimes argued that the more `chaotic' nature of small batches leads to beneficial regularization.
However, due to the complex interaction between the batch size and the stopping rule, the effect of this is not clear-cut.}.
On the issue of minibatch size,
we use the Keras default of 32.

\subsection{Results}
\label{ss.results}

Table \ref{tab:learning}  reports  the neural network's 
average relative in-sample performance, and its comparison 
to the out-of-sample test set performance
with the above described conservative stopping rule.
For each dimension, the corresponding $\mu$ value is the 
average of 30 runs and $\sigma$ is the standard deviation.
We keep the data size of $N=100,000$ and the network architecture of three hidden 
layers of width 10 fixed.
Even though with this rule the stochastic gradient descent is stopped
quite early, there is substantial overperformance increasing with dimension.

\begin{table}[ht]
  \centering
  \begin{tabular}{c|rr|rr}
    \multirow{2}{*}{dims} & \multicolumn{2}{c|}{$p_{in}$ (\%)} 
    & \multicolumn{2}{c}{$p_{in} - p_{out}$ (\%)} \\
& $\mu\ \quad$ & $\sigma\ \quad$ & $\mu\ \quad$ & $\sigma\ \quad$ \\
    \hline
    100 & 10.12820 & 1.09290 & 23.67080 & 2.01177 \\
    85 & 8.38061 & 1.35575 & 20.16440 & 2.30489 \\
    70 & 7.32720 & 0.86458 & 15.62060 & 1.94043 \\
    55 & 5.05783 & 0.81518 & 10.93950 & 1.54431 \\
    40 & 3.74648 & 0.62588 & 7.91105 & 1.32581 \\
    25 & 2.11501 & 0.43845 & 4.58954 & 0.88461 \\
    10 & 0.53982 & 0.34432 & 1.46138 & 0.39078 \\
    \end{tabular}
  \caption{\label{tab:learning}Average  
  relative in-sample performance, and its comparison 
  to the out-of-sample performance
  with the above described conservative stopping rule.  
 Everything is in \% with training size of $N=100,000$ and
 three hidden layers of width 10.  
 The $\mu$ value
is the average of 30 runs
 and $\sigma$ is the standard deviation.
}
\end{table}

To isolate the impact of the dimension,
in the second 
experiment, we keep all parameters except the width of layers 
 as before. The last two hidden layers 
 again have width 10.  But
the  width of the first hidden layer is adjusted so that
 the number of parameters  is equal to that of a neural network
 with three hidden layers of width 10 and input layer 
 of dimension as in column `parameters-equivalent'.  There are three groups with
 parameters-equivalent dimensions of 40, 70 and 100.  For example
 in the group with parameters-equivalent dimension 70, in the row with actual
 dimension 70, all layers have width 10.  But in that group, 
 the networks for the actual dimensions of 40 and 10 have wider first layer so
that they all have the same number of parameters.  Table \ref{tab:adjdims} 
 also shows a clear increase of overlearning with dimension.
 Although, the architecture is not exactly same, we believe that 
 this experiment shows that the apparent dimensional dependence is not 
 simply due to
 the increase in the number of parameters.

\begin{table}[ht]
  \centering
  \begin{tabular}{c|c|rr|rr}
    \multirow{2}{*}{dims} & \multirow{2}{*}{\shortstack[c]{params-\\ equiv}} 
    & \multicolumn{2}{c|}{$p_{in}$ (\%)} & \multicolumn{2}{c}{$p_{in} - p_{out}$ (\%)} \\
    & & $\mu\ \quad$ & $\sigma\ \quad$ & $\mu\ \quad$ & $\sigma\ \quad$ \\
    \hline
    100 & 100 & 10.12820 & 1.09290 & 23.67080 & 2.01177 \\
    70 & 100 & 8.86214 & 1.45962 & 21.65000 & 3.12209 \\
    40 & 100 & 7.28550 & 1.19811 & 15.27540 & 2.10167 \\
    10 & 100 & 1.99793 & 0.54664 & 4.18041 & 1.22285 \\
    \hline
    70 & 70 & 7.32720 & 0.86458 & 15.62060 & 1.94043 \\
    40 & 70 & 5.67500 & 0.84644 & 12.45610 & 1.90450 \\
    10 & 70 & 1.50328 & 0.93772 & 3.46245 & 1.19606 \\
    \hline
    40 & 40 & 3.74648 & 0.62588 & 7.91105 & 1.32581 \\
    10 & 40 & 1.13566 & 0.65512 & 2.84677 & 0.78069 \\
    \end{tabular}
  \caption{\label{tab:adjdims} 
  All other parameters except the width of layers 
  are as in Table \ref{tab:learning}. The last two hidden layers
  again have width 10 and
the  width of the first hidden layer is adjusted so that
 the number of parameters  is equal to that of a neural network
 with three hidden layers of width 10 and the number 
 of dimension is as in the parameters-equivalent column.
 }
\end{table}

We have also implemented an aggressive optimization
by running the algorithm for 100 and 200 
epochs in 100 dimensions without a stopping rule
with other parameters as in Table \ref{tab:learning}.
In these experiments trained actions $A_{k,n}$
are closer to optimal actions $A^*_{k,n}$
and Theorem \ref{t.learning} predicts a larger
overperformance.  Indeed,
Table \ref{tab:utilmax_perf} shows that,
overlearning is quite substantial even with a training size of $100,000$
and there is a noticeable deterioration in the out-of-sample performance.
\begin{table}[ht]
  \centering
  \begin{tabular}{c|rr|rr}
    \multirow{2}{*}{epochs} & \multicolumn{2}{c|}{$p_{in}$ (\%)} 
    & \multicolumn{2}{c}{$p_{in} - p_{out}$ (\%)} \\
& $\mu\ \quad$ & $\sigma\ \quad$ & $\mu\ \quad$ & $\sigma\ \quad$ \\
  \hline
    200 & 30.3161 & 2.46850 & 315.875 & 540.4750 \\
    100 & 25.8374 & 1.72027 & 111.553 & 41.5841 \\
	\end{tabular}
  \caption{\label{tab:utilmax_perf}Longer iterations performance
  in 100 dimensions. Based on 15 runs.
  Especially the 200 epoch runs show signs of a heavy tail, as expected with high degrees of overlearning.
  All other parameters as in Table \ref{tab:learning}.}
\end{table}

Finally, Table \ref{tab:large_ssize} illustrates 
the convergence proved in Section \ref{ss.convergence}.
In 100 dimensions
we increase the size of the training data  from $100,000$ to
twenty-fold keeping all the other parameters as in the first experiment.
The results show a remarkable improvement in the accuracy
demonstrating the power of dynamic deep empirical risk minimization.

\begin{table}[ht]
  \centering
  \begin{tabular}{c|c|rr|rr}
    \multirow{2}{*}{\shortstack[c]{sample\\ size}} & \multirow{2}{*}{dims} & \multicolumn{2}{c|}{$p_{in}$ (\%)} 
    & \multicolumn{2}{c}{$p_{in} - p_{out}$ (\%)} \\
        & & $\mu\ \quad$ & $\sigma\ \quad$ & $\mu\ \quad$ & $\sigma\ \quad$ \\
    \hline
    2,000,000 & 100 & 0.49597 & 0.11849 & 1.12846 & 0.32184 \\
    1,000,000 & 100 & 1.14094 & 0.14640 & 2.39532 & 0.27415 \\
    500,000 & 100 & 2.36352 & 0.20154 & 5.24018 & 0.81235 \\
    250,000 & 100 & 4.41388 & 0.37928 & 10.02040 & 1.45355 \\
  \end{tabular}
  \caption{\label{tab:large_ssize}Performance for larger sample sizes.  Based on 15 runs.  All other parameters as in Table~\ref{tab:learning}.}
\end{table}

\section{Conclusions}

Dynamic deep empirical risk minimization is a highly
effective computational tool for 
many investment or hedging problems, or
more generally, decision 
making under uncertainty.  It can handle 
general random structures in high dimensions
and complex dynamics
with ease.  The simplicity of
the algorithm and the recent advances in the
training of deep neural networks are key to these
properties.  
By both theoretical results and numerical experiments, 
we have demonstrated that the hypothesis spaces
can overlearn the data, and consequently,
construct forward-looking feedback actions.
As shown in Theorem~\ref{t.ant},
the optimization step is able to 
by-pass the adaptedness requirement,
and the trained actions may become
non-adapted to the flow
of information on the training data.
Thus, in-sample-value 
approximates the strictly smaller minimum given
by the anticipative controls.

The estimates proved in Theorem \ref{t.rademacher}
show that 
overlearning is negligible when the data set is sufficiently large
compared to the complexity of the networks.
As one needs sufficient complexity of the
neural networks to achieve appropriate 
accuracy, the size of the training set is
critical.  When a particular model is assumed,
dynamic simulation during the training is an effective
way of increasing the size of the data.  Without such a model
and small size data, one must design simulation
mechanisms based on the given data
and  use it dynamically during the optimization. 
Therefore, for an efficient application 
of these methods in  data-driven envoroinments,
calibration of complex models to the market data
is a crucial step.
 
Although this approach is particularly valuable in high dimensions,
overlearning becomes easier for those problems,
thus requiring richer
training sets, as  clearly demonstrated by the numerical studies
reported in Section \ref{s.numerics}.
For optimal control, an in-depth study of this dependence,
both numerically and theoretically, remains an interesting question.
The proof of the asymptotic overlearning result Theorem \ref{t.learning}
provides an initial insight indicating that the average distance between
the data points and the regularity of the networks are
important for a better
understanding of this dependence.  Indeed, the Rademacher complexity
which is present in the 
upper bound \eqref{e.rate}
is also influenced by both of them.
Closely related covering numbers 
providing an upper for the Rademacher complexity
\citep[cf.][Lemma 27.4]{SSB} could also be useful in
understanding this dependence.
\vspace{10pt}

\noindent{\large\bf Acknowledgments.}
\noindent 
 {The authors thank Professors E, Fan and Han of Princeton University,
 and two anonymous referees for their valuable comments. }
\vskip 0.2in

\bibliographystyle{abbrvnat}
\bibliography{overlearning}

\begin{appendix}

\section{Asymptotic Overlearning}
\label{appendixA}

In this section, we prove an extension of
Theorem \ref{t.learning}  proved for distinct training data.
Although this is a natural assumption which holds for instance, 
when $Z^{(i)}$ are drawn independently
from an atomless distribution,
we provide this
extension to further facilitate our understanding of overlearning.

Fix a training
set $\cL_n$.  Let
$\cK=\{\cK^{(1)}, \ldots, \cK^{(m)}\}$
be a partition of $\cL_n$ satisfying:
\begin{itemize}
\item $\cK^{(j)}$ are disjoint subsets of $\cL_n$;
\item $\cup_j \cK^{(j)} = \cL_n$;
\item if $z \in \cK^{(j)}$ for some $j$, and if there is a trajectory 
$\hat{z} \in \cL_n$ and $ t \in \cT$ such that
$z_t = \hat{z}_t$, then $\hat{z}\in \cK^{(j)}$.
\end{itemize}
There are
partitions satisfying the  above conditions
and one can even define and would like to use
the maximal partition satisfying above conditions.
As this is tangential to
the main thrust of the paper,  we do not pursue it here.
When the data is distinct, the maximal partition is $\cK^{(i)}=\{Z^{(i)}\}$
and we are back in the setting of Theorem \ref{t.learning}.

For a constant control $\alpha=(\alpha_0,\ldots,\alpha_{T-1}) \in \cA^\cT$,
and for $j=1,\ldots,m$,   define
$$
\oell(\alpha,j):= \frac{1}{\left| \cK^{(j)}\right|} \sum_{z \in \cK^{(j)}} 
\ell(\alpha,z).
$$
Let  $\ase{(j)} \in \cA^\cT$ be an $\eps$-minimizer of $\oell(\cdot,j)$.
Analogously to $L^*_n$ defined in \reff{l.equal}, define
$$
\overline{L}^*_n:=
\frac{1}{m} \sum_{j=1}^m \inf_{\alpha \in \cA^\cT} \oell(\alpha,j)
=\lim_{\eps \downarrow 0} \frac{1}{m} \sum_{j=1}^m \oell(\ase{(j)},j).
$$

For $z \in \cL_n$, let $j(z)$ be the unique index so that 
$z \in \cK^{(j(z))}$.
We now follow the arguments of Theorem \ref{t.learning} \emph{mutadis mutandis}
to show that the hypothesis spaces can approximate the function
$$
\overline{a}^*_\eps(z):= \ase(j(z)), \quad z \in \cL_n.
$$
This implies
 the following extension of the overlearning result Theorem \ref{t.learning}.

\begin{lemma}
Let $A^*_{k,n}$ be as in \eqref{e.vkn}.  Under 
 the {density  assumption {\rm{(}}c.f.~Assumption \ref{a.approximate}{\rm{)}}},
$$
 \lim_{k \to \infty} L(A^*_{k,n};\cL_n) \le  \overline{L}^*_n.
$$
\end{lemma}

When the partition $\cK$ of $\cL_n$ is non-trivial
and  
if the number of partitions $m$ is large, then we 
may have $ \overline{L}^*_n < v^*$ and 
consequently potential overlearning.
The robust approach used in \cite{BSS,BSS2}
and also in \cite{EK,BDT}, essentially groups the
elements of the training set into a small number of sets
and identifies them by a representative element of these sets.
If we then partition this processed data, this would result
in a small  number of partitions and the overlearning 
will not be possible even with modest size training sets. 

{\section{Markov Decision Processes}
\label{appendixB}
An essential structural requirement 
of our formulation is that
the random process $Z$ driving the
state dynamics is independent of control.  
Although many control problems may not be initially
expressed that way, if their dynamics is known,
they could still be reformulated to fit into 
our framework.  Here we discuss one such central example
of controlled Markov chains to illustrate this point.}

Let $\{Y_t\}_{t \in \cT}$ be a $\cS:=\{s_1,\ldots,s_M\}$ valued 
controlled Markov chain.  Suppose that the 
transition probabilities are given by
$$
p(t,y,\tilde{y},a):= \P(Y_{t+1} = y\, | \, Y_t =\tilde{y},\ A_t=a),
\qquad
y, \tilde{y} \in \cS, \ a \in \cA, \ t \in \cT.
$$
For a given  $Y_0$ and $\widehat{\Phi}$, 
the classical control problem  is
to
$$
\text{minimize} \quad \E[\widehat{\Phi}(Y_1,\ldots,Y_T,A_0,\ldots,A_{T-1})]
$$
over all feedback controls $A$.

To reformulate this problem, we first
introduce a random process  $Z=(Z_1,\ldots,Z_T)$
satisfying,
$$
\P(Z_t = y)=\frac1M, \qquad
\forall \ y \in \cS,\ t=1,\ldots,T.
$$
Let $X_0=(Y_0,1)$, 
$\cX:= \cS \times [0,1]$
and for $ t=1,\ldots,T$,  set $X_t=:(X^{(1)}_t,X^{(2)}_t)=(Z_t,R_t)$,
where the second component $R$ is the Radon--Nikodym process defined recursively by,
$$
R_0=1,\quad
\text{and}
\quad
R_{t+1}= M\, p(t, Z_{t+1}, X^{(1)}_{t},A_t)\ R_t,\qquad
t \in \cT.
$$
We rewrite the equations for the components 
of $X$ as 
$$
X_{t+1}= \left(Z_{t+1},\  M\, p(Z_{t+1},X^{(1)}_{t},A_t)\, X^{(2)}_t\right),
$$
verifying that the dynamics of $X$ is in the form
\eqref{e.state}. Also, it can be directly shown that
$$
\E[\widehat{\Phi}(Y_1,\ldots,Y_T,A_0,\ldots,A_{T-1})]
= 
\E[\widehat{\Phi}(X^{(1)}_{1},\ldots,X^{(1)}_{T},A_0,\ldots,A_{T-1})\, X^{(2)}_{T} ].
$$
Hence, the original problem is equivalent to the 
control of the process $X$ with 
$$
\Phi(X,A,Z):= \widehat{\Phi}(X^{(1)}_{1},\ldots,X^{(1)}_{T},A_0,\ldots,A_{T-1})\, X^{(2)}_{T}.
$$

\begin{remark}
\label{r.rcontrol}
{\rm{ Randomized controls that are widely
used in problems with learning can also be 
included in our framework.  However,
we should note that
we are primarily interested in constructing the 
optimal feedbacks and
in our framework they always exist
and randomization is not needed.
}}

{\rm{Consider the above 
optimization problem with the controlled
Markov chain $Y \in \cS$
and  a finite control set
$\cA=\{a^{1},\ldots,a^{n}\}$. A randomized control
$U_t$  at time $t$ 
is a probability on the control set
$\cA$. Thus, it takes values in the simplex,
$$
 \Delta_{n-1}:= \{ u=(u^1,\ldots,u^n)\in \R_+^n\ |\
u^1+\ldots+u^n=1\ \}.
$$
The components of $U_t=(U_t^1,\ldots,U_t^n) \in \Delta_{n-1}$
correspond to the probability that
a particular control is used in that step, i.e.,
$U^i_t=\P(A_t=a^{i})$.
Let $p(y,\tilde{y},a^i)$ be the transition probabilities
of $Y_t$ when control $a^{i}$ is chosen.
For a given $U=(U^1,\ldots,U^n) \in  \Sigma_{n-1}$, consider the
transition probabilities for the pair
$X_t:=(Y_t, a_t) \in \cS \times \cA=: \cX$,
\begin{align*}
\widehat{p}\, ((y,a^i), (\tilde{y},a^j), U) & := 
\P(Y_{t+1} = y,\, a_{t+1}=a^{i}\, | \, Y_t =\tilde{y},\, U_t=U)\\
&= p(y,\tilde{y},a^{i})\, U^i,\qquad
\text{for } y, \tilde{y} \in \cS, \ i,j=1,\ldots,n.
\end{align*}
So if the process $X_t$ is at the state $(\tilde{y}, a^j)$
and the randomized control $U$ is chosen,
then a new control $a^i$ is chosen using the probability $U$.
Then, the transition of the $Y$ component
is decided by the transition distribution $p(\cdot,\tilde{y},a^i)$.
In particular, the control $a^j$ chosen a step earlier does not
impact the dynamics of this step.
This construction turns the randomized controls into  a
standard controlled Markov chain with an enlarged state space
$\cX=\cS \times \cA$ and a control set $\Delta_{n-1}$.
One can then use procedure outlined
above with the Radon--Nikodym
process to reformulate it with
an uncontrolled
random process.}}\end{remark}
\end{appendix}
\vskip 0.1in

\end{document}